\pdfoutput=1

\documentclass[11pt]{article}

\usepackage{acl}

\usepackage{times}
\usepackage{latexsym}
\usepackage{microtype}

\usepackage{dirtytalk} %
\usepackage{booktabs}
\usepackage{amsfonts}
\usepackage{amsmath}
\usepackage{amsthm}
\usepackage{enumitem}
\usepackage{makecell}
\usepackage{diagbox}
\usepackage{multirow}
\usepackage{cleveref}
\usepackage{subcaption}
\usepackage{bm}

\usepackage{thmtools,thm-restate}

\usepackage{bbm}

\usepackage{xcolor}
\usepackage{soul}
\newcommand{\hlc}[2][yellow]{{\sethlcolor{#1}\hl{#2}}}

\usepackage{algorithm, algorithmicx}

\usepackage[noend]{algpseudocode}%

\DeclareMathOperator*{\argmax}{arg\,max}

\crefname{example}{example}{examples}
\Crefname{example}{Example}{Examples}
\DeclareCaptionType{example}[Example][List of examples]
\crefname{algorithmB}{algorithm}{algorithms}
\Crefname{algorithmB}{Algorithm}{Algorithms}
\DeclareCaptionType{algorithmB}[Algorithm][List of algorithms]

\usepackage[pdftex]{graphicx}

\usepackage[T1]{fontenc}

\usepackage[utf8]{inputenc}

\usepackage{microtype}

\usepackage{inconsolata}

\usepackage{graphicx}
\usepackage{todonotes}

\definecolor{TodoColor}{HTML}{F0A0F0}

\title{Distributional Properties of Subword Regularization}

\author{Marco Cognetta \\
  Tokyo Institute of Technology \\
  \href{mailto:cognetta.marco@gmail.com}{\texttt{cognetta.marco@gmail.com}}
  \And
  Vilém Zouhar \\
  ETH Zürich \\
  \href{mailto:vzouhar@inf.ethz.ch}{\texttt{vzouhar@inf.ethz.ch}}
  \And
  Naoaki Okazaki \\
  Tokyo Institute of Technology \\
  \href{mailto:okazaki@c.titech.ac.jp}{\texttt{okazaki@c.titech.ac.jp}} 
}

\begin{document}
\maketitle

\begin{abstract}
Subword regularization, used widely in NLP, improves model performance by reducing the dependency on exact tokenizations, augmenting the training corpus, and exposing the model to more unique contexts during training. BPE and MaxMatch, two popular subword tokenization schemes, have stochastic \textit{dropout} regularization variants.
However, there has not been an analysis of the distributions formed by them.
We show that these stochastic variants are heavily biased towards a small set of tokenizations per word. 
If the benefits of subword regularization are as mentioned, we hypothesize that biasedness artificially limits the effectiveness of these schemes.
Thus, we propose an algorithm to uniformly sample tokenizations that we use as a drop-in replacement for the stochastic aspects of existing tokenizers, and find that it improves machine translation quality.%

\end{abstract}

\smallskip

\section{Introduction}

Tokenization is the first stage in almost all natural language processing pipelines, where raw text is transformed into a format that is understood by the model.
Modern neural models use \textit{subword} tokenization, which represents text as a sequence of subword units drawn from a subword vocabulary (e.g., \texttt{decompositional} $\rightarrow$ \texttt{de composition al}). 
Popular subword tokenization schemes are BPE \cite{sennrich-etal-2016-neural}, MaxMatch/WordPiece \cite{wu2016googles}, and UnigramLM \cite{kudo-2018-subword}.
Unintentionally, the downstream models are thus not conditioned on the raw text, but rather the \textit{exact tokenization} of the text.
During training, subword regularization (where static tokenizations are replaced with sampled tokenizations) is often used to break the dependency on the exact tokenization.
It also serves as data augmentation, and improves performance in a variety of downstream tasks.

\begin{table}[t]
\renewcommand\arraystretch{0.75}
\centering
\resizebox{\linewidth}{!}{
\begin{tabular}{>{\hspace{-1mm}\small}l<{\hspace{-1mm}}>{\hspace{-2mm}\small}r}
\toprule
\multicolumn{2}{c}{\small \bf \makecell{BPE-Dropout $p = 0.1$}} \\
\midrule
\texttt{to ken ization}  & 97.77\% \\
\texttt{to ke n ization}         & 1.89\% \\
\texttt{to k en ization}         & 0.25\% \\
\texttt{to ken iz ation}         & 0.04\% \\
\texttt{t oken ization}  & 0.03\% \\
\texttt{to k en iz ation}        & 0.01\% \\
\texttt{to ke n iz ation}        & 0.01\% \\
\texttt{to ken i z ation}        & <~0.01\% \\
\bottomrule
\end{tabular}
\begin{tabular}{>{\hspace{-1mm}\small}l<{\hspace{-1mm}}>{\hspace{-2mm}\small}r}
\toprule
\multicolumn{2}{c}{\small \bf \makecell{MaxMatch-Dropout $p = 0.3$}} \\
\midrule
\texttt{to ken ization}  & 34.29\% \\
\texttt{t oken ization}  & 14.66\% \\
\texttt{to ke n ization}         & 10.48\% \\
\texttt{to ken iz ation}         & 7.21\% \\
\texttt{t oke n ization}         & 4.39\% \\
\texttt{to k en ization}         & 3.15\% \\
\texttt{t oken iz ation}         & 3.05\% \\
\texttt{to ke n iz ation}        & 2.14\% \\
\bottomrule
\end{tabular}
}
\vspace{-2mm}
\captionof{example}{
The most frequently observed tokenizations of the word \texttt{tokenization} and their empirical frequencies with BPE-Dropout and MaxMatch-Dropout.
}
\label{ex:bpe_dropout_20}
\vspace{-5mm}
\end{table}

There are two main types of stochastic tokenizers: those which learn a distribution from text (e.g., UnigramLM) and those which inject randomness by corrupting the tokenization scheme (e.g., BPE-Dropout, \citealp{provilkov-etal-2020-bpe} and MaxMatch-Dropout, \citealp{hiraoka-2022-maxmatch}).
In our work, we focus on the latter, for which no prior study of the resulting distributions exists.
BPE- and MaxMatch-Dropout add randomness \textit{post hoc} into the underlying deterministic tokenization, and the distributions they produce are essentially unrelated to the text distribution.
We find that these distributions are heavily biased, in that they do not produce uniform tokenization distributions (see \Cref{ex:bpe_dropout_20}).

Despite them working well in practice, there is no reason to believe that the distributions formed by BPE- and MaxMatch-Dropout are ``good'' for training.
However, there are reasons to believe that a different strategy, uniform sampling, would be \textit{better} for training, as it would increase the amount of regularization and augmentation injected into the training process.
We experiment with replacing the stochastic aspects of BPE- and MaxMatch-Dropout with one which samples uniformly at random from all possible tokenizations, and find that it improves modeling quality on several translation tasks.

\section{Motivation} \label{sec:motivation}

Though stochastic tokenization is known to improve model quality, it remains unclear which tokenization distribution is the best. BPE- and MaxMatch-Dropout, which induce \textit{unlearned} distributions (probabilities that are not chosen by a learning algorithm), are a natural way of injecting randomness into the underlying tokenization algorithm \textit{post hoc}. However, empirically, we see that they both induce heavily \textit{biased} distributions, and hypothesize that an \textit{unbiased} stochastic tokenizer would be universally better.
This hypothesis is based on three explanations for subword regularization's effectiveness:

\noindent
\textbf{1) Regularization}
Subword regularization regularizes the model by breaking the dependency on a single, canonical tokenization.
As shown in \Cref{ex:bpe_dropout_20}, BPE- and MaxMatch-Dropout allocate most of their probability mass to only a few tokenizations for a given input.
A tokenizer that uniformly samples from the distribution will expose the model to a greater variety of unique tokenizations of the same input text during training.

\noindent
\textbf{2) Augmentation}
Subword regularization acts as data augmentation by increasing the number of unique inputs that are seen during model training.
An unbiased tokenization sampler will produce more unique tokenizations of the same input than a biased sampler.

\noindent
\textbf{3) Efficiency}
Subword regularization increases the tokenizer's \textit{efficiency} in the information-theoretic sense,\footnote{Rényi efficiency is defined as
$H_{\alpha}(p_{\mathcal{V}})/{\log(|\mathcal{V}|)}$,
where $H_\alpha$ is Rényi entropy, $\mathcal{V}$ is a subword vocabulary, and $p_{\mathcal{V}}(w)$ is the unigram probability of subword in the tokenized  corpus.
}
which is a quality shown to be well correlated with downstream task performance \cite{gutierrez-vasques-etal-2021-characters,zouhar-etal-2023-tokenization}.
A tokenizer with unbiased sampling will generally have higher efficiency than a biased one.

\section{Subword Tokenization}

Subword tokenizers are typically deterministic in that the same character sequence will result in the same tokenized output sequence.
\textit{Stochastic} variants were developed to allow for sampling tokenizations, which has been shown to improve model quality and robustness in a variety of NLP tasks.
We briefly introduce three common subword tokenization schemes and their stochastic variants, which all share the same formalization.

\paragraph{Formalization.}
Deterministic tokenization maps words $w \in \Sigma^+$ to a sequence of subwords from a finite vocabulary $\Sigma \subseteq \mathcal{V} \subset \Sigma^*$ as $t(w) \in \mathcal{V}^+$.
An important part of tokenization is that it is lossless---a tokenization of an input can be inverted to recover the original word.
For example, $t(\texttt{tokenization})$ = \texttt{to ken ization} but importantly $t^{-1}(\texttt{to ken ization})$ = \texttt{tokenization}.

In contrast, stochastic tokenization is not a one-to-one mapping, but rather a probability distribution function $T_w$ for each word $w$.
This assigns each tokenization $\bar{w}$ a probability $T_w(\bar{w}) \in [0,1]$.
Continuing \Cref{ex:bpe_dropout_20}, $T_w(\texttt{to ken ization})$ = 0.98 and $T_w(\texttt{to ke n ization})$ = 0.02.
During application of the tokenizer, the specific tokenization of $w$ is sampled from the distribution $T_w$.

\paragraph{BPE and Dropout.}
BPE forms a vocabulary by iteratively merging the most frequently cooccurring pair of tokens in the corpus\footnote{Appendix \ref{apx:algorithms}, \Cref{alg:bpe_initialization}.} \cite{sennrich-etal-2016-neural}.
During inference, the sequence of learned merges is applied greedily to new text (Algorithm \ref{alg:bpe_inference}).
To implement dropout, a probability $p$ is introduced, and the \hlc[pink!50]{highlighted statement} randomly removes candidate merges \cite{provilkov-etal-2020-bpe}.

\begin{figure}[t]
{\small
\hrule \vspace{1mm}
\textbf{Inputs}: Word $w \in \Sigma^+$, (Ordered) Merges $\mu$\newline
\textbf{Output}: Tokenized sequence $t \in \mathcal{V}^+$
\hrule
\begin{algorithmic}[1]
\State $\varphi \gets \langle (w_i, w_{i+1}) | (w_i, w_{i+1}) \in \mu~$~\hlc[pink!50]{$\wedge~\textsc{Rand}() > p$}$ \rangle$
\For{$\varphi \ne \emptyset$}
    \State $(x, y) \gets \argmax_{\mu} \varphi$ \Comment{Ordered by $\mu$}
    \State $w \gets \textsc{Replace}((x, y) \rightarrow xy, w)$
    \State $\varphi \gets \langle (w_i, w_{i+1}) |~(w_i, w_{i+1}) \in \mu~$~\hlc[pink!50]{$\wedge~\textsc{Rand}() > p$}$ \rangle$
\EndFor
\State \Return $w$\label{alg:inference_return}
\end{algorithmic}
}
\hrule
\vspace{1mm}
\captionof{algorithmB}{%
    BPE Inference (\hlc[pink!50]{with dropout})
    }
\label{alg:bpe_inference}
\vspace{-7mm}
\end{figure}

\paragraph{MaxMatch and Dropout.}
Given a subword vocabulary,\footnote{We use the standard WordPiece training algorithm as described by \cite{SchusterN12}.} MaxMatch tokenizes text from left to right by iteratively selecting the longest matching subword, shown in \Cref{alg:maxmatch}.
MaxMatch-Dropout randomly discards matching subwords and falls back to shorter ones via \hlc[pink!50]{the condition} on Line \ref{line:maxmatch_dropout_statement} \cite{hiraoka-2022-maxmatch}.

\begin{figure}[t]
{
\small
\hrule \vspace{1mm}
\textbf{Inputs}: Word $w \in \Sigma^+$, Vocabulary $\mathcal{V}$\newline
\textbf{Output}: Tokenized sequence $t \in \mathcal{V}^+$
\hrule    
\begin{algorithmic}[1]
\State $t \gets \bm{[}~\bm{]}, i \gets 1$
\While{$i \le |w|$}
    \State $z \gets w_i$
    \For{$j \in 1 \dots \max_{v \in \mathcal{V}} |v|$} 
        \If{$w_{i \,\ldots\, i+j} \in \mathcal{V}~$\hlc[pink!50]{$\wedge~\textsc{Rand}() > p$}} \label{line:maxmatch_dropout_statement}
            \State $z \gets w_{i \,\ldots\, i+j}$
        \EndIf
    \EndFor
    \State $t \xleftarrow[]{\text{append}}z,~i \gets i+|z|$
\EndWhile
\State \Return $t$
\end{algorithmic}
\hrule
}
\vspace{1mm}
\captionof{algorithmB}{%
MaxMatch Inference (\hlc[pink!50]{with dropout})    
}
\label{alg:maxmatch}
\vspace{-7mm}
\end{figure}

\paragraph{UnigramLM.}
UnigramLM \cite{kudo-2018-subword} introduced the concept of subword regularization.
It learns a vocabulary and unigram probabilities for each token in the vocabulary according to some loss function over the training corpus.
The tokenization of an input text is sampled from the probability distribution induced by the model and some temperature $\alpha$.\footnote{Note that $\alpha = 0$ yields the uniform distribution, but this would sample a tokenization of the entire sentence, rather than word-by-word, which harms model quality.} 
We do not explore the distributional properties of UnigramLM here, since they are highly corpus dependant.

\paragraph{Dropout Distributions.}
\Cref{ex:bpe_dropout_20} and \Cref{apx:tokenizations} show empirical probabilities for certain values of $p$ in BPE- and MaxMatch-Dropout.
While BPE- and MaxMatch-Dropout were not designed to form (or even claimed to be) unbiased distributions, here we concretely show that their distributions are biased, under mild conditions.

\begin{restatable}{lemma}{bpelemma}\label{lma:bpe}
Let $\mathcal{B} = (\mathcal{V}, \mu)$ be a BPE tokenizer such that there exists $(a, b), (b, b), (b, c) \in \mu$ with $(a, b) >_{\mu} (b, b) >_{\mu} (b, c)$ and $abb, bbc, abbc \notin \mathcal{V}$. Then, there exists a word $w \in \Sigma^+$ for which the distribution of the dropout tokenizer $\mathcal{B}'(w)$ is non-uniform for any $p$.
\end{restatable}

\begin{restatable}{lemma}{mmlemma}\label{lma:mm}
Let $\mathcal{M}$ be a MaxMatch tokenizer over vocabulary $\mathcal{V}$, such that $\Sigma \subset \mathcal{V}$ and there exists a token $v \in \mathcal{V} \setminus \Sigma$ which is a proper prefix of some other token $z = vy \in \mathcal{V}$. Then, there exists a word $w \in \Sigma^+$ for which the distribution of the dropout tokenizer $\mathcal{M}'(w)$ is non-uniform for any $p$.
\end{restatable}

\begin{figure}[htbp]
    \centering
    \begin{subfigure}{\columnwidth}
        \centering
        \vspace{-5mm}
        \includegraphics[width=\columnwidth]{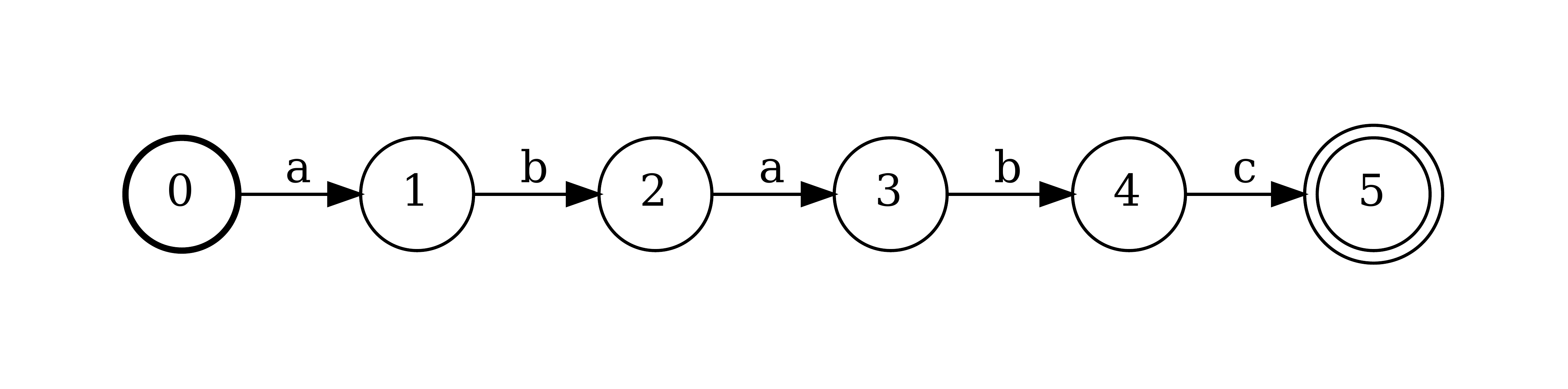}
        \vspace{-10mm}
        \caption{An automaton $\mathcal{A}$ representing \texttt{ababc}.}\label{fig:input_automaton}
    \end{subfigure}
    \begin{subfigure}{\columnwidth}
        \centering
        \includegraphics[scale=0.1]{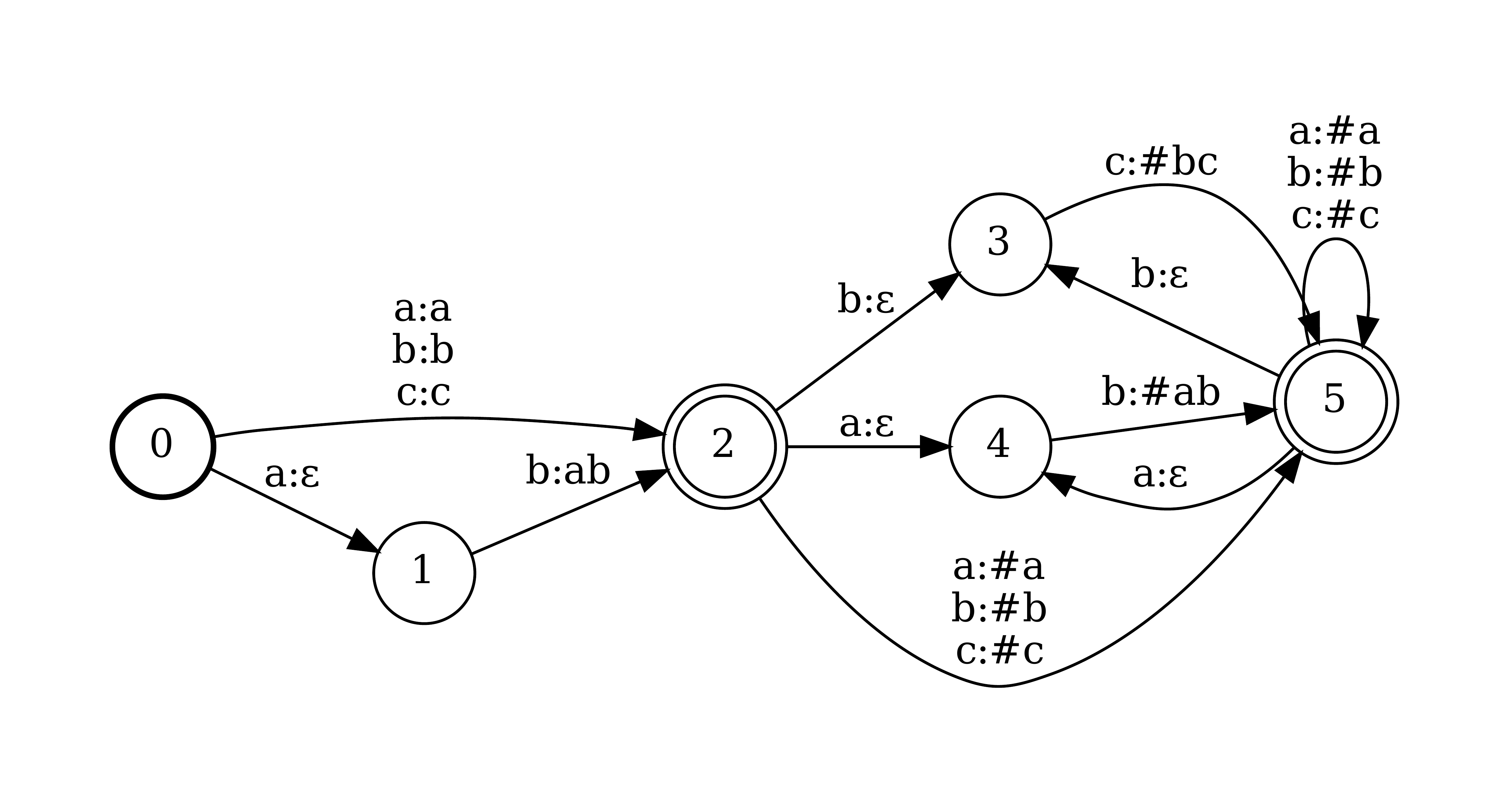}
        \captionsetup{justification=centering}
        \vspace{-10mm}
        \caption{A transducer $\mathcal{T}$ for the subword vocabulary \\ \texttt{ \{a, b, c, ab, \#a, \#b, \#c, \#ab, \#bc\}}.}\label{fig:transducer}
    \end{subfigure}
    \begin{subfigure}{\columnwidth}
        \centering
        \includegraphics[width=\columnwidth]{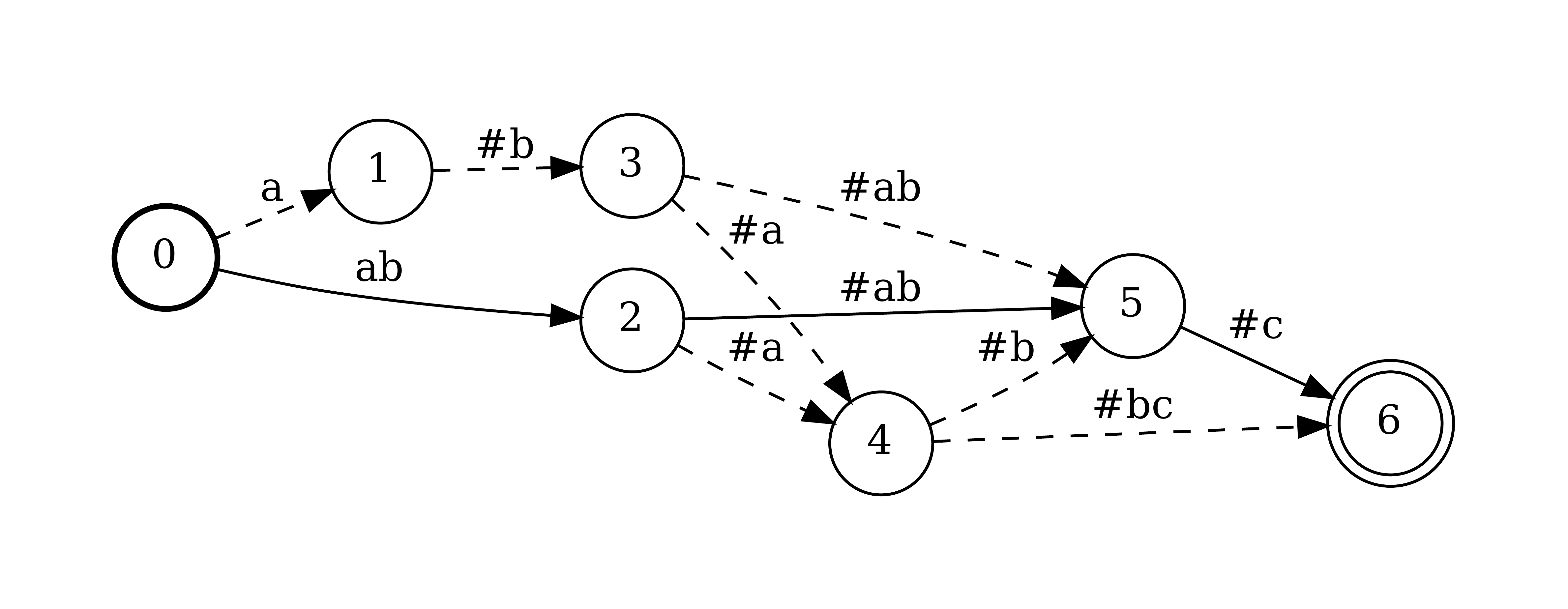}
        \vspace{-10mm}
        \caption{A lattice, $\mathcal{A} \circ \mathcal{T}$, of all possible tokenizations of \texttt{ababc}.}\label{fig:lattice}
    \end{subfigure}

    \caption{Uniformly sampling tokenizations from $\mathcal{A} \circ \mathcal{T}$.}
    \label{fig:sampling_example}
\end{figure}

\section{Uniformly Sampling Tokenizations}

Given a subword vocabulary, we first produce a character-to-subword finite-state transducer representing it. Encoding an input word as a linear finite-state automaton and composing it with this transducer produces a lattice which encodes all possible tokenizations. Since the word length is finite, this lattice must be acyclic, and we can sample paths from it using Algorithm \ref{alg:unbiased_dag} \cite{lavrovexchange}.
An example of the transducer construction, composition, and sampling is shown in \Cref{fig:sampling_example}.

Given a baseline BPE or MaxMatch tokenizer, we implement our uniform sampling tokenizer by constructing a subword transducer from its subword vocabulary and selecting a dropout probability $p$. 
During training, a word
is tokenized via uniform sampling with probability $p$ and via the deterministic tokenizer with probability $1-p$, as shown in Algorithm \ref{alg:uniform_sample_tokenizer}.

One of the reasons for the success of subword regularization is that they expose the model to a more diverse set of tokenizations (\Cref{sec:motivation}).
Figure \ref{fig:word_imbalance_div} shows that across any choice of $p$, even with far fewer samples, a much more diverse set of tokenizations for a given word is observed when using Uniform Sampling compared to Dropout, indicating that a model will be exposed to a far greater number of unique contexts during training. %

\begin{figure}
{
\hrule \vspace{1mm}
\small
\textbf{Inputs}: Sentence $s$, Tokenizer $T$, Probability $p$\newline
\textbf{Output}: Tokenized sequence $t \in \mathcal{V}^+$
\hrule
\begin{algorithmic}[1]
\State $t \gets \bm{[}~\bm{]}$
\For{$w \in s$}
    \If{$\textsc{Rand}() < p$}
        \State $t \xleftarrow[]{\text{extend}}\textsc{UniformSample}(w)$
    \Else
        \State $t \xleftarrow[]{\text{extend}}T(w)$
    \EndIf
\EndFor
\State \Return $t$
\end{algorithmic}
\hrule
}
\vspace{1mm}
\captionof{algorithmB}{%
    Uniform Sampling Tokenization.
}
\label{alg:uniform_sample_tokenizer}
\vspace{-5mm}
\end{figure}

\begin{figure}[htbp]
\includegraphics[width=\linewidth]{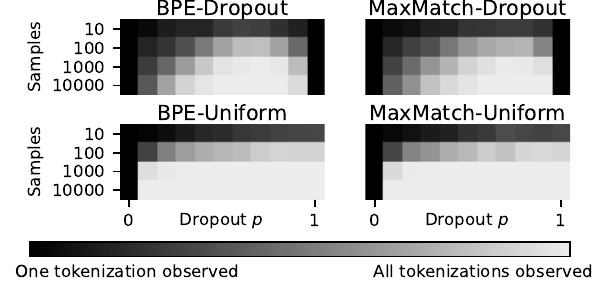}
\caption{The number of unique, observed tokenizations of a word with $N$ samples and dropout $p$.}
\label{fig:word_imbalance_div}
\end{figure}

\section{Experiments}

We use English$\leftrightarrow$German, English$\leftrightarrow$Romanian, and English$\leftrightarrow$French as our translation tasks. 
For each language pair, we train a baseline BPE and MaxMatch tokenizer with the same vocabulary size and use them to build Dropout and Uniform Sampler variants so that the vocabulary between a baseline tokenizer and its stochastic variants is exactly the same and only the tokenization distribution is different. We include a UnigramLM tokenizer with the same vocabulary size as a learned-distribution baseline.%
We use the same underlying transformer model (\Cref{apx:architecture}) for each language pair, and only change the embedding and decoding layers, according to the choice of tokenizer. We compare tokenizer efficiency (via  \href{https://github.com/zouharvi/tokenization-scorer}{\texttt{tokenization-scorer}}),
\textsc{BLEU} \citep{papineni-etal-2002-bleu,post-2018-call}, \textsc{chrF} \citep{popovic-2015-chrf}, and COMET\textsubscript{DA-22} \citep{rei-etal-2022-comet} by averaging the results of three experimental runs per model. 
We use $p=0.1$ for BPE-Dropout, $p=0.3$ for MaxMatch-Dropout, and $\alpha=0.3$ for UnigramLM. %
For Uniform Sampling, we use $p=0.1$ and $0.25$, which were chosen as an estimate of the frequency that a non-canonical tokenization of word in BPE- and MaxMatch-Dropout was sampled, respectively. Thus, we should expect Uniform Samplers to have roughly the same amount of non-canonically-tokenized-words in a corpus as BPE- and MaxMatch-Dropout, so the salient difference is the variety of tokenizations.

The results are shown in Table \ref{tab:results_sample}. In nearly every translation metric, Uniform Sampling outperforms BPE- and MaxMatch-Dropout. However, curiously, Uniform Sampling does not always have higher efficiency than BPE- or MaxMatch-Dropout (but is always higher than the baseline), as Uniform Sampling guarantees maximum entropy at the \textit{word}-tokenization level, which does not necessarily translate to the \textit{global}-tokenization entropy. 

There is only one metric (EN$\rightarrow$DE, BPE, \textsc{chrF}) where a Uniform Sampling model is not the best. However, in that same case, the Uniform Sampler improved upon the BPE-Dropout model by 0.8 BLEU, which is nearly as much as the BPE-Dropout improved upon the BPE baseline.
In addition, the +0.61 increase in COMET\textsubscript{DA-22} corresponds to a 82\% agreement accuracy with humans \cite{kocmi2024navigating}. In the EN$\rightarrow$RO pair, Uniform Sampling models were the best across all metrics and underlying tokenizers. Further, Unigram Sampling consistently outperforms UnigramLM both in terms of raw translation quality metrics and improvement over the deterministic baseline. These results extend to our full experimental results (Appendix \ref{apx:fullresults}) and support our hypothesis that an unbiased tokenizer should generally outperform biased dropout tokenizers.

\begin{table}[h]
        \begin{subtable}{\columnwidth}\centering
        {
        \resizebox{\columnwidth}{!}{%
            \begin{tabular}{l|cccc} \toprule
                Tokenizer & Efficiency & \textsc{BLEU} & \textsc{chrF} & \textsc{COMET} \\ \midrule
                BPE                         & 0.4636 & 28.44        & 55.80  & 76.22  \\
                BPE + Dropout ($p{=}0.1$)     & 0.4747 & 29.37        & \textbf{56.63}  & 77.51  \\
                BPE + Uniform ($p{=}0.1$)     & 0.4731 & 30.05        & 56.37  & \textbf{78.12}  \\
                BPE + Uniform ($p{=}0.25$)     & 0.4719 & \textbf{30.16}        & 56.47  & 78.08  \\ \midrule
                MaxMatch                         & 0.4584   & 28.41             & 55.97 & 76.57 \\
                MaxMatch + Dropout ($p{=}0.3$)     & 0.4530   & 29.13             & 56.43 & 77.38 \\
                MaxMatch + Uniform ($p{=}0.1$)     & 0.4657   & 29.18             & 56.43 & \textbf{77.76} \\
                MaxMatch + Uniform ($p{=}0.25$)     & 0.4633   & \textbf{29.43}    & \textbf{56.57} & 77.62 \\ \midrule
                Unigram ($\alpha{=}1$)     & 0.4452 & 28.40 & 55.93 & 76.66 \\ 
                Unigram ($\alpha{=}0.3$)   & 0.3796 & 28.97 & 56.33 & 77.44 \\ \bottomrule
        \end{tabular}}\caption{English$\rightarrow$German (source+target dropout)}}
        \end{subtable}
        \begin{subtable}{\columnwidth}\centering
        {
        \resizebox{\columnwidth}{!}{%
            \begin{tabular}{l|cccc} \toprule
                Tokenizer & Efficiency & \textsc{BLEU} & \textsc{chrF} & \textsc{COMET} \\ \midrule
                BPE                              & 0.4524       & 23.56 & 53.20 & 81.03  \\
                BPE + Dropout ($p{=}0.1$)        & 0.4614       & 23.98 & 53.70 & 81.90  \\
                BPE + Uniform ($p{=}0.1$)        & 0.4594       & 23.83 & 53.67 & 82.00  \\
                BPE + Uniform ($p{=}0.25$)       & 0.4647       & \textbf{24.13} & \textbf{53.73} & \textbf{82.20}  \\ \midrule
                MaxMatch                         & 0.4476       & 23.52 & 53.23 & 81.17 \\
                MaxMatch + Dropout ($p{=}0.3$)   & 0.4578       & 23.95 & 53.70 & 81.98 \\
                MaxMatch + Uniform ($p{=}0.1$)   & 0.4528       & \textbf{24.32} & \textbf{53.90} & \textbf{82.11} \\ 
                MaxMatch + Uniform ($p{=}0.25$)  & 0.4563       & 24.10 & 53.87 & 82.06 \\ \midrule
                Unigram ($\alpha{=}1$)                & 0.4338       & 23.68 & 53.37 & 81.28 \\
                Unigram ($\alpha{=}0.3$)              & 0.4284       & 24.17 & 53.87 & 82.00 \\\bottomrule
        \end{tabular}}\caption{English$\rightarrow$Romanian (source only dropout)}}
        \end{subtable}

\caption{
Experimental results for EN$\rightarrow$DE and EN$\rightarrow$RO. In each block, we compare a baseline tokenizer with its dropout and uniform sampling variants.
Each group has the same vocabulary and differs only in the tokenization distribution.
The best performing model for each baseline and metric is \textbf{bolded}.
The full results for all languages is in Appendix \ref{apx:fullresults}.
}\label{tab:results_sample}
\vspace{-2mm}
\end{table}

\section{Conclusion}

We investigate the distributions induced by BPE- and MaxMatch-Dropout, two popular subword regularization schemes.
We hypothesize and show that BPE- and MaxMatch-Dropout are suboptimal in that they form heavily \hlc[gray!40]{biased distributions}.
We introduce a Uniform Sampler tokenizer, which\linebreak \hlc[gray!40]{guarantees uniform distributions} and consistently outperforms BPE- and MaxMatch-Dropout on machine translation tasks.

\smallskip
\noindent
\textbf{Future work.}
Uniform Sampling is uniform at the \textit{word} level, but past research suggests that uniformity at the \textit{global} unigram level is desired.
Therefore, algorithms could be designed to directly optimize global uniformity.
Further investigations should reconcile how both Uniform Sampling and UnigramLM improve performance despite their opposing motivations (higher/lower entropy).

\section*{Limitations}

We did not establish statistical significance for our results, but note that the trend hold across language pairs, tokenizers, and metrics. We did not do substantial hyperparameter searching for vocabulary size or dropout rates, but rather used values that commonly appear in the literature. It is possible that some trends in our results may change with different choices of tokenization hyperparameters.

We also did not experiment with extremely-low resource settings (our smallest setting, EN$\leftrightarrow$DE has 150k sentence pairs), or very large settings (our largest, EN$\leftrightarrow$FR, has 2M sentence pairs). Additionally, in our largest case, the improvement seen by Uniform Sampling are less consistent and less significant. However, this is in line with prior research that shows the diminishing effectiveness of subword regularization as the corpus size increases.

\bibliography{custom}
\clearpage

\appendix

\section{Training Details}\label{apx:architecture}

We use \texttt{fairseq} \citep{ott-etal-2019-fairseq} for language modeling, HuggingFace's \href{https://github.com/huggingface/tokenizers}{\tt tokenizers} library for our underlying BPE and MaxMatch tokenizers, and OpenFst \cite{openfst} for the subword lattice construction. For UnigramLM, we used SentencePiece \cite{kudo-richardson-2018-sentencepiece}. We used \texttt{fairseq}'s \texttt{transformer\_iwslt\_de\_en} architecture for EN$\leftrightarrow$DE, and the baseline \texttt{transformer} architecture for EN$\leftrightarrow$RO and EN$\leftrightarrow$FR. The hyperparameters and optimizer configuration are given in Tables \ref{tab:iwslt_transformer}, \ref{tab:transformer}, and \ref{tab:optimizer}.
Our datasets were:
\begin{itemize}[left=0mm,noitemsep,topsep=0mm]
    \item EN$\leftrightarrow$DE: 160k, IWSLT14 \cite{iwslt14}
    \item EN$\leftrightarrow$RO: 600k, WMT16 \cite{wmt16}
    \item EN$\leftrightarrow$FR: 2M, Europarl \cite{europarl}
\end{itemize}

\begin{table}[!h]
\centering
        \resizebox{0.9\columnwidth}{!}{%
        \begin{tabular}{l|c} \toprule
        \makecell[l]{Vocabulary Sizes  (src, tgt)} & \makecell{EN$\leftrightarrow$DE: (10k, 10k)}\\
        Embedding Dimension   & 512 \\ 
        FFN Dimension         & 1024 \\ 
        Number of Heads        & 4    \\
        Number of Layers       & 6    \\ 
        Dropout                & 0.3  \\ \bottomrule
        \end{tabular}}\caption{The \texttt{transformer\_iwslt\_en\_de} architecture, used for the English$\leftrightarrow$German task.}\label{tab:iwslt_transformer}
\end{table}

\begin{table}[!h]
\centering
        \resizebox{0.9\columnwidth}{!}{%
        \begin{tabular}{l|c} \toprule
        \makecell[l]{Vocabulary Sizes  (src, tgt)} & \makecell{EN$\leftrightarrow$RO: (14k, 14k) \\ EN$\leftrightarrow$FR: (30k, 30k)} \\ 
        Embedding Dimension   & 512 \\ 
        FFN Dimension         & 2048 \\ 
        Number of Heads        & 6    \\
        Number of Layers       & 8    \\ 
        Dropout                & 0.1  \\ \bottomrule
        \end{tabular}}\caption{The \texttt{transformer} architecture, used for the English$\leftrightarrow$Romainan and English$\leftrightarrow$French tasks.}\label{tab:transformer}
\end{table}

\begin{table}[!h]
    \centering
        \begin{tabular}{ll} \toprule
        Optimizer & ADAM \\
        $\beta_1, \beta_2$     & $(0.9, 0.98)$ \\
        Learning Rate          & $5\times 10^{-4}$ \\
        Warmup                 & 4000 steps    \\
        Scheduler       & Inverse Square Root    \\
        Tokens-per-batch & 8192 \\
        Patience         & \makecell[l]{EN$\leftrightarrow$DE: 8 \\ EN$\leftrightarrow$RO: 10 \\ EN$\leftrightarrow$FR: 5} \\ \bottomrule
        \end{tabular}
        \caption{The optimizer parameters, used for all tasks.}\label{tab:optimizer}
\end{table}

\newpage

\section{Algorithms} \label{apx:algorithms}
\begin{figure}[!h]
{
\small
\hrule\vspace{1mm}
\textbf{Inputs}: Corpus $\mathcal{C}$, Alphabet $\Sigma$, Target size $n$,\newline
\textbf{Outputs}: Vocabulary $\mathcal{V}$, Merges $\mu$
\hrule
\begin{algorithmic}[1]
\State $\mathcal{V} \gets \Sigma$
\For{$i \in 1 \dots n$}
    \State $(x, y) \gets \argmax\limits_{a, b \in \mathcal{V}} \textsc{count}({(a, b), \mathcal{C}})$
    \State $\mathcal{V} \gets \mathcal{V} \cup \{ xy \}$
    \State $\mu \gets \mu \cup \langle (x, y) \rangle$
    \State $\mathcal{C} \gets \textsc{Replace}({(x, y) \rightarrow xy, \mathcal{C}})$
\EndFor
\State \Return $\mathcal{V}, \mu$\label{alg:training_return}
\end{algorithmic}
}
\hrule\medskip
\captionof{algorithmB}{%
    BPE Training.
}
\label{alg:bpe_initialization}
\end{figure}
\begin{figure}[!h]
    \small
    \hrule\smallskip
    \newcommand{\cur}{\ensuremath{\textsc{cur}}}
    \textbf{Inputs}: Directed Acyclic Graph $D$,\newline
    \textbf{Outputs}: Path $\pi$, Path-probability $p$
    \hrule
    \begin{algorithmic}[1]
            \State $\pi \gets \bm{[}~\bm{]}$
            \State $p \gets 1$
            \State $\cur \gets q_{start}$
            \While{$\cur$ is not final}
                \State $(w, q) \sim \textsc{Uniform}(\textsc{Adj}(\cur))$
                \State $\textsc{Append}(\pi, (\cur, w, q))$
                \State $p \gets p \times \frac{1}{\Call{Deg}{\cur}}$
                \State $\cur \gets q$
            \EndWhile
            \State \Return{$\pi, p$}
   \end{algorithmic}
   \hrule\medskip
\captionof{algorithmB}{Biased DAG Sampling.}\label{alg:biased_dag}
\end{figure}
\begin{figure}[!h]
    \small
    \hrule\smallskip
    \newcommand{\cur}{\ensuremath{\textsc{cur}}}
    \textbf{Inputs}: Directed Acyclic Graph $D$,\newline
    \textbf{Output}: Path $\pi$
    \hrule
    \begin{algorithmic}[1]
            \State $p_{min} = \prod\limits_{q \in D} \frac{1}{\textsc{deg}(q)}$

            \State $(\pi, p) \sim \textsc{BiasedSample}(D)$
            
            \While{$\Call{Rand}{\null} > \frac{p_{min}}{p}$}
                \State $(\pi, p) \sim \textsc{BiasedSample}(D)$
            \EndWhile
            \State \Return{$\pi$}
    \end{algorithmic}
    \hrule\medskip
\captionof{algorithmB}{Unbiased DAG Sampling.}\label{alg:unbiased_dag}
\end{figure}

\newpage

\section{Proofs}
\bpelemma*
\begin{proof}
Since $abb, bbc, abbc \notin \mathcal{V}$, then $(ab, b), (a, bb), (b, bc) \notin \mu$.
Consider the word $abbc$. There are 5 possible tokenizations $[a, b, b, c]$, $[a, b, bc]$, $[a, bb, c]$, $[ab, b, c]$, $[ab, bc]$. We proceed by case analysis and compute the probability of each, given a dropout probability $p$.
\begin{itemize}[left=0mm]
    \item $[a, b, b, c] \rightarrow p^3$
    \begin{itemize}
        \item Merges $(a, b), (b, b)$, and $(b, c)$ must be dropped which has probability $p{\times}p{\times}p{=}p^3$.
    \end{itemize}
    \item $[a, b, bc] \rightarrow p^2(1-p)$
    \begin{itemize}
        \item Merges $(a, b)$, and $(b, b)$ must be dropped, but not $(b, c)$.
    \end{itemize}
    \item $[a, bb, c] \rightarrow p(1-p)$
    \begin{itemize}
        \item Merge $(a, b)$ must be dropped but not $(b, b)$. The merge $(b, c)$ is irrelevant as $(b, b) >_{\mu} (b, c)$.
    \end{itemize}
    \item $[ab, b, c] \rightarrow (1-p)p$
    \item $[ab, bc] \rightarrow (1-p)^2$
\end{itemize}

To be uniform, $p^3 = p^2(1-p) = p(1-p) = (1-p)^2 = \frac{1}{5}$, which does not exist. Hence, there is no $p$ such that $\mathcal{B}'(abbc)$ is uniform.
\end{proof}

\mmlemma*
\begin{proof}
Consider the distribution of tokenizations $\mathcal{M}'(z)$, under which the probability of $z$ being the final tokenization is $(1-p)$.  Let the total number of tokenizations be $n$, and assume the distribution is unbiased. The probability of the tokenization $[v, y_1, y_2, \dots, y_k]$ is $(1-p)p^k$. Thus, the distribution is only unbiased if $(1-p) = (1-p)p^k = \frac{1}{n}$. Since there are at least $3$ tokenizations $[z], [z_1, z_2, \dots, z_n], \text{and } [v, y_1, y_2, \dots, y_k]$, this is a contradiction.
\end{proof}

\begin{figure}[htbp]
\includegraphics[width=\linewidth]{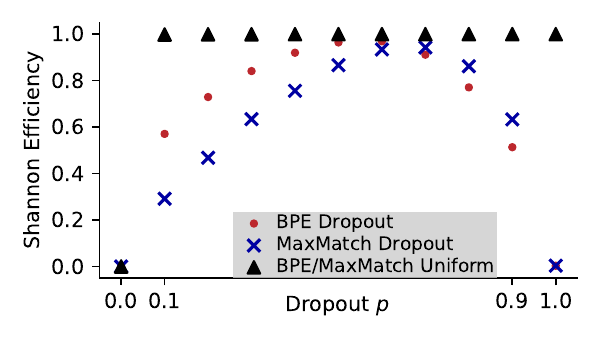}
\caption{Distribution uniformity measured by Shannon Efficiency (higher=more uniform; excludes the canonical form, which usually takes up most of the probability mass). Our Uniform Sampling versions (both for BPE and MaxMatch) \textit{guarantee} balanced sampling of tokenizations, which is not true for the standard Dropout versions whose balance depends non-linearly on the dropout rate $p$.}
\label{fig:word_imbalance_eff}
\end{figure}

\onecolumn
\section{Full Experiments}\label{apx:fullresults}
\begin{table*}[!h]
    \begin{subtable}{.5\linewidth}\centering
    {
        \resizebox{\columnwidth}{!}{%
            \begin{tabular}{l|cccc} \toprule
                Tokenizer & Efficiency & \textsc{BLEU} & \textsc{chrF} & \textsc{COMET} \\ \midrule
                BPE                         & 0.4636 & 33.66        & 57.10  & 79.30  \\
                BPE + Dropout ($p{=}0.1$)     & 0.4747 & 35.06        & 58.07  & 80.51  \\
                BPE + Uniform ($p{=}0.1$)     & 0.4731 & 35.03        & 57.97  & 80.46  \\
                BPE + Uniform ($p{=}0.25$)    & 0.4719 & \textbf{35.22}        & \textbf{58.13}  & \textbf{ 80.57}  \\ \midrule
                MaxMatch                           & 0.4584        & 33.85  & 57.17 & 79.48 \\
                MaxMatch + Dropout ($p{=}0.3$)      & 0.4530        & 34.92  & 57.87 & 80.37 \\
                MaxMatch + Uniform ($p{=}0.1$)      & 0.4657        & 35.17  & 58.10 & 80.60 \\
                MaxMatch + Uniform ($p{=}0.25$)       & 0.4633        & \textbf{35.32}  & \textbf{58.13} & \textbf{80.71} \\ \midrule
                Unigram ($\alpha{=}1$)     & 0.4452 & 33.37 & 56.77 & 79.43 \\ 
                Unigram ($\alpha{=}0.3$)   & 0.3796 & 34.24 & 57.70 & 80.31 \\ \bottomrule
        \end{tabular}}\caption{German$\rightarrow$English (source+target dropout)}}
    \end{subtable}
    \vspace{0.1cm}
    \begin{subtable}{.5\linewidth}\centering
    {
        \resizebox{\columnwidth}{!}{%
            \begin{tabular}{l|cccc} \toprule
                Tokenizer & Efficiency & \textsc{BLEU} & \textsc{chrF} & \textsc{COMET} \\ \midrule
                BPE                         & 0.4636 & 28.44        & 55.80  & 76.22  \\
                BPE + Dropout ($p{=}0.1$)     & 0.4747 & 29.37        & \textbf{56.63}  & 77.51  \\
                BPE + Uniform ($p{=}0.1$)     & 0.4731 & 30.05        & 56.37  & \textbf{78.12}  \\
                BPE + Uniform ($p{=}0.25$)     & 0.4719 & \textbf{30.16}        & 56.47  & 78.08  \\ \midrule
                MaxMatch                         & 0.4584   & 28.41             & 55.97 & 76.57 \\
                MaxMatch + Dropout ($p{=}0.3$)     & 0.4530   & 29.13             & 56.43 & 77.38 \\
                MaxMatch + Uniform ($p{=}0.1$)     & 0.4657   & 29.18             & 56.43 & \textbf{77.76} \\
                MaxMatch + Uniform ($p{=}0.25$)     & 0.4633   & \textbf{29.43}    & \textbf{56.57} & 77.62 \\ \midrule
                Unigram ($\alpha{=}1$)     & 0.4452 & 28.40 & 55.93 & 76.66 \\ 
                Unigram ($\alpha{=}0.3$)   & 0.3796 & 28.97 & 56.33 & 77.44 \\ \bottomrule
        \end{tabular}}\caption{English$\rightarrow$German (source+target dropout)}}
    \end{subtable}
    \begin{subtable}{.5\linewidth}\centering
    {
        \resizebox{\columnwidth}{!}{%
            \begin{tabular}{l|cccc} \toprule
                Tokenizer & Efficiency & \textsc{BLEU} & \textsc{chrF} & \textsc{COMET} \\ \midrule
                BPE                              & 0.4034       & 40.86 & 64.57 & 86.39  \\
                BPE + Dropout ($p{=}0.1$)        & 0.4137       & 40.96 & 64.57 & 86.50  \\
                BPE + Uniform ($p{=}0.1$)        & 0.4139       & \textbf{41.10} & \textbf{64.70} & \textbf{86.52}  \\
                BPE + Uniform ($p{=}0.25$)       & 0.4259       & 40.86 & 64.57 & 86.36  \\ \midrule
                MaxMatch                         & 0.4003       & 41.02 & \textbf{64.70} & 86.48 \\
                MaxMatch + Dropout ($p{=}0.3$)   & 0.4186       & 40.88 & 64.57 & 86.47 \\
                MaxMatch + Uniform ($p{=}0.1$)   & 0.4094       & \textbf{41.04} & \textbf{64.70} & \textbf{86.54} \\ 
                MaxMatch + Uniform ($p{=}0.25$)  & 0.4194       & 40.80 & 64.50 & 86.38 \\ \midrule
                Unigram ($\alpha{=}1$)                & 0.3801       & 40.59 & 64.43 & 86.27 \\
                Unigram ($\alpha{=}0.3$)              & 0.3773       & 40.71 & 64.53 & 86.36 \\\bottomrule
        \end{tabular}}\caption{French$\rightarrow$English (source only dropout)}}
    \end{subtable}
    \begin{subtable}{.5\linewidth}\centering
    {
        \resizebox{\columnwidth}{!}{%
            \begin{tabular}{l|cccc} \toprule
                Tokenizer & Efficiency & \textsc{BLEU} & \textsc{chrF} & \textsc{COMET} \\ \midrule
                BPE                              & 0.4034       & 41.27 & 65.77 & 86.86  \\
                BPE + Dropout ($p{=}0.1$)        & 0.4137       & 41.45 & 65.90 & \textbf{87.08}  \\
                BPE + Uniform ($p{=}0.1$)        & 0.4139       & \textbf{41.54} & \textbf{65.93} & 87.04  \\
                BPE + Uniform ($p{=}0.25$)       & 0.4259       & 41.35 & 65.83 & 87.03  \\ \midrule
                MaxMatch                         & 0.4003       & 41.38 & 65.87 & 87.00 \\
                MaxMatch + Dropout ($p{=}0.3$)   & 0.4186       & 41.24 & 65.80 & 86.95 \\
                MaxMatch + Uniform ($p{=}0.1$)   & 0.4094       & \textbf{41.44} & \textbf{65.93} & \textbf{87.07} \\ 
                MaxMatch + Uniform ($p{=}0.25$)  & 0.4194       & 41.22 & 65.77 & 86.93 \\ \midrule
                Unigram ($\alpha{=}1$)                & 0.3801       & 40.47 & 65.27 & 86.36 \\
                Unigram ($\alpha{=}0.3$)              & 0.3773       & 40.15 & 65.00 & 86.08 \\\bottomrule
        \end{tabular}}\caption{English$\rightarrow$French (source only dropout)}}
    \end{subtable}
    \begin{subtable}{.5\linewidth}\centering
    {
        \resizebox{\columnwidth}{!}{%
            \begin{tabular}{l|cccc} \toprule
                Tokenizer & Efficiency & \textsc{BLEU} & \textsc{chrF} & \textsc{COMET} \\ \midrule
                BPE                              & 0.4524       & 30.81 & 57.30 & 77.77  \\
                BPE + Dropout ($p{=}0.1$)        & 0.4614       & \textbf{32.13} & 58.17 & 79.50  \\
                BPE + Uniform ($p{=}0.1$)        & 0.4594       & 31.92 & 58.13 & \textbf{79.64}  \\
                BPE + Uniform ($p{=}0.25$)       & 0.4647       & 31.85 & \textbf{58.23} & 79.54  \\ \midrule
                MaxMatch                         & 0.4476       & 31.01 & 57.23 & 78.03 \\
                MaxMatch + Dropout ($p{=}0.3$)   & 0.4578       & 31.90 & 58.13 & 79.63 \\
                MaxMatch + Uniform ($p{=}0.1$)   & 0.4528       & \textbf{32.02} & \textbf{58.37} & \textbf{79.81} \\ 
                MaxMatch + Uniform ($p{=}0.25$)  & 0.4563       & 31.83 & 58.33 & 79.74 \\ \midrule
                Unigram ($\alpha{=}1$)                & 0.4338       & 30.34 & 56.97 & 77.74 \\
                Unigram ($\alpha{=}0.3$)              & 0.4284       & 31.53 & 58.07 & 79.40 \\\bottomrule
        \end{tabular}}\caption{Romanian$\rightarrow$English (source only dropout)}}
    \end{subtable}
    \vspace{0.1cm}
    \begin{subtable}{.5\linewidth}\centering
    {
        \resizebox{\columnwidth}{!}{%
            \begin{tabular}{l|cccc} \toprule
                Tokenizer & Efficiency & \textsc{BLEU} & \textsc{chrF} & \textsc{COMET} \\ \midrule
                BPE                              & 0.4524       & 23.56 & 53.20 & 81.03  \\
                BPE + Dropout ($p{=}0.1$)        & 0.4614       & 23.98 & 53.70 & 81.90  \\
                BPE + Uniform ($p{=}0.1$)        & 0.4594       & 23.83 & 53.67 & 82.00  \\
                BPE + Uniform ($p{=}0.25$)       & 0.4647       & \textbf{24.13} & \textbf{53.73} & \textbf{82.20}  \\ \midrule
                MaxMatch                         & 0.4476       & 23.52 & 53.23 & 81.17 \\
                MaxMatch + Dropout ($p{=}0.3$)   & 0.4578       & 23.95 & 53.70 & 81.98 \\
                MaxMatch + Uniform ($p{=}0.1$)   & 0.4528       & \textbf{24.32} & \textbf{53.90} & \textbf{82.11} \\ 
                MaxMatch + Uniform ($p{=}0.25$)  & 0.4563       & 24.10 & 53.87 & 82.06 \\ \midrule
                Unigram ($\alpha{=}1$)                & 0.4338       & 23.68 & 53.37 & 81.28 \\
                Unigram ($\alpha{=}0.3$)              & 0.4284       & 24.17 & 53.87 & 82.00 \\\bottomrule
        \end{tabular}}\caption{English$\rightarrow$Romanian (source only dropout)}}
    \end{subtable}
    \bigskip
    \begin{subtable}{.5\linewidth}\centering
    {
        \resizebox{\columnwidth}{!}{%
            \begin{tabular}{l|cccc} \toprule
                Tokenizer & Efficiency & \textsc{BLEU} & \textsc{chrF} & \textsc{COMET} \\ \midrule
                BPE                              & 0.4524       & 30.81 & 57.30 & 77.77  \\
                BPE + Dropout ($p{=}0.1$)        & 0.4672       & 32.48 & 58.87 & 80.29  \\
                BPE + Uniform ($p{=}0.1$)        & 0.4615       & 32.46 & 58.87 & 80.26  \\
                BPE + Uniform ($p{=}0.25$)       & 0.4562       & \textbf{32.83} & \textbf{59.07} & \textbf{80.71}  \\ \midrule
                MaxMatch                         & 0.4476       & 31.01 & 57.23 & 78.03 \\
                MaxMatch + Dropout ($p{=}0.3$)   & 0.4465       & 32.89 & 59.03 & \textbf{80.69} \\
                MaxMatch + Uniform ($p{=}0.1$)   & 0.4544       & 32.83 & 58.97 & 80.36 \\ 
                MaxMatch + Uniform ($p{=}0.25$)  & 0.4484       & \textbf{33.03} & \textbf{59.13} & 80.65 \\ \midrule
                Unigram ($\alpha{=}1$)                & 0.4338       & 30.29 & 57.07 & 0.7774 \\
                Unigram ($\alpha{=}0.3$)              & 0.4061       & 32.30 & 58.70 & 0.8011 \\ \bottomrule
        \end{tabular}}\caption{Romanian$\rightarrow$English (source+target dropout)}}
    \end{subtable}
    \vspace{0.1cm}
    \begin{subtable}{.5\linewidth}\centering
    {
        \resizebox{\columnwidth}{!}{%
            \begin{tabular}{l|cccc} \toprule
                Tokenizer & Efficiency & \textsc{BLEU} & \textsc{chrF} & \textsc{COMET} \\ \midrule
                BPE                              & 0.4524       & 23.56 & 53.20 & 81.03  \\
                BPE + Dropout ($p{=}0.1$)        & 0.4672       & \textbf{24.85} & \textbf{54.30} & 82.71  \\
                BPE + Uniform ($p{=}0.1$)        & 0.4615       & 24.78 & \textbf{54.30} & \textbf{83.03}  \\
                BPE + Uniform ($p{=}0.25$)       & 0.4562       & 24.77 & 54.00 & 82.67  \\ \midrule
                MaxMatch                         & 0.4476       & 23.52 & 53.23 & 81.17 \\
                MaxMatch + Dropout ($p{=}0.3$)   & 0.4465       & 25.02 & 54.20 & 82.67 \\
                MaxMatch + Uniform ($p{=}0.1$)   & 0.4544       & 24.77 & \textbf{54.40} & \textbf{83.08} \\ 
                MaxMatch + Uniform ($p{=}0.25$)  & 0.4484       & \textbf{25.16} & 54.33 & 83.00 \\ \midrule
                Unigram ($\alpha{=}1$)                & 0.4338       & 23.61 & 53.23 & 0.8114 \\
                Unigram ($\alpha{=}0.3$)              & 0.4061       & 24.61 & 54.13 & 0.8274 \\ \bottomrule
        \end{tabular}}\caption{English$\rightarrow$Romanian (source+target dropout)}}
    \end{subtable}
    \caption{The main results of machine translation performance (average across 3 seeds). In almost all cases the Uniform sampling yields the best results. }
    \label{tab:main_results}
\end{table*}
\newpage

\section{Examples of Distributions from Data}\label{apx:tokenizations}
\begin{table}[htbp]
\centering
\resizebox{\linewidth}{!}{
\begin{tabular}{>{\hspace{-1mm}\small}p{2cm}<{\hspace{-1mm}}r}
\toprule
\multicolumn{2}{c}{\small \bf \makecell{BPE-Dropout $p{=}0.1$}} \\
\midrule
\texttt{something}       & 96.50\% \\
\texttt{some thing}      & 1.60\% \\
\texttt{so met hing}     & 1.52\% \\
\texttt{so m et hing}    & 0.16\% \\
\texttt{so me thing}     & 0.09\% \\
\texttt{som eth ing}     & 0.03\% \\
\texttt{s ome thing}     & 0.03\% \\
\texttt{so m eth ing}    & 0.03\% \\
\texttt{somet hing}      & 0.01\% \\
\texttt{some th ing}     & 0.00\% \\
\bottomrule
\end{tabular}
\begin{tabular}{>{\hspace{-1mm}\small}p{2cm}<{\hspace{-1mm}}>{\hspace{-2mm}}r}
\toprule
\multicolumn{2}{c}{\small \bf \makecell{MaxMatch-Dropout $p{=}0.3$}} \\
\midrule
\texttt{something}       & 69.85\% \\
\texttt{somet hing}      & 14.79\% \\
\texttt{somet hi n g}    & 4.54\% \\
\texttt{some thing}      & 4.35\% \\
\texttt{somet h ing}     & 1.31\% \\
\texttt{some th ing}     & 0.93\% \\
\texttt{som eth ing}     & 0.91\% \\
\texttt{some t hing}     & 0.42\% \\
\texttt{somet h in g}    & 0.38\% \\
\texttt{some th in g}    & 0.29\% \\
\bottomrule
\end{tabular}
\hspace{5mm}
\begin{tabular}{>{\hspace{-1mm}\small}p{2cm}<{\hspace{-1mm}}r}
\toprule
\multicolumn{2}{c}{\small \bf \makecell{BPE-Dropout $p{=}0.1$}} \\
\midrule
\texttt{started}         & 97.56\% \\
\texttt{star ted}        & 2.15\% \\
\texttt{star te d}       & 0.11\% \\
\texttt{start ed}        & 0.07\% \\
\texttt{st ar ted}       & 0.05\% \\
\texttt{st art ed}       & 0.02\% \\
\texttt{s ta r ted}      & 0.02\% \\
\texttt{star t ed}       & 0.01\% \\
\texttt{s ta r t ed}     & 0.00\% \\
\texttt{s t art ed}      & 0.00\% \\
\bottomrule
\end{tabular}
\begin{tabular}{>{\hspace{-1mm}\small}p{2cm}<{\hspace{-1mm}}>{\hspace{-2mm}}r}
\toprule
\multicolumn{2}{c}{\small \bf \makecell{MaxMatch-Dropout $p{=}0.3$}} \\
\midrule
\texttt{started}         & 69.97\% \\
\texttt{start ed}        & 14.88\% \\
\texttt{start e d}       & 6.21\% \\
\texttt{star ted}        & 4.37\% \\
\texttt{star te d}       & 1.33\% \\
\texttt{st art ed}       & 0.91\% \\
\texttt{star t ed}       & 0.42\% \\
\texttt{s ta r ted}      & 0.39\% \\
\texttt{st art e d}      & 0.39\% \\
\texttt{st ar ted}       & 0.26\% \\
\bottomrule
\end{tabular}
}

\bigskip

\resizebox{\linewidth}{!}{
\begin{tabular}{>{\hspace{-1mm}\small}p{2cm}<{\hspace{-1mm}}r}
\toprule
\multicolumn{2}{c}{\small \bf \makecell{BPE-Dropout $p{=}0.1$}} \\
\midrule
\texttt{percent}         & 73.54\% \\
\texttt{per c ent}       & 8.83\% \\
\texttt{per cent}        & 8.12\% \\
\texttt{perce nt}        & 7.84\% \\
\texttt{per ce nt}       & 0.88\% \\
\texttt{p er c ent}      & 0.18\% \\
\texttt{p er ce n t}     & 0.12\% \\
\texttt{p er cent}       & 0.11\% \\
\texttt{pe r cent}       & 0.10\% \\
\texttt{per ce n t}      & 0.09\% \\
\bottomrule
\end{tabular}
\begin{tabular}{>{\hspace{-1mm}\small}p{2cm}<{\hspace{-1mm}}>{\hspace{-2mm}}r}
\toprule
\multicolumn{2}{c}{\small \bf \makecell{MaxMatch-Dropout $p{=}0.3$}} \\
\midrule
\texttt{percent}         & 69.93\% \\
\texttt{perce nt}        & 14.61\% \\
\texttt{perce n t}       & 6.40\% \\
\texttt{per cent}        & 4.53\% \\
\texttt{pe r cent}       & 1.33\% \\
\texttt{per ce nt}       & 0.90\% \\
\texttt{per ce n t}      & 0.39\% \\
\texttt{p er cent}       & 0.38\% \\
\texttt{per c ent}       & 0.37\% \\
\texttt{pe r ce nt}      & 0.26\% \\
\bottomrule
\end{tabular}
\hspace{5mm}
\begin{tabular}{>{\hspace{-1mm}\small}p{2cm}<{\hspace{-1mm}}r}
\toprule
\multicolumn{2}{c}{\small \bf \makecell{BPE-Dropout $p{=}0.1$}} \\
\midrule
\texttt{different}       & 82.44\% \\
\texttt{dif fe rent}     & 8.74\% \\
\texttt{diff ere nt}     & 8.05\% \\
\texttt{differ ent}      & 0.29\% \\
\texttt{dif fe re nt}    & 0.20\% \\
\texttt{dif f ere nt}    & 0.13\% \\
\texttt{dif fer ent}     & 0.07\% \\
\texttt{d iff ere nt}    & 0.03\% \\
\texttt{d if fer ent}    & 0.03\% \\
\texttt{di ff er ent}    & 0.01\% \\
\bottomrule
\end{tabular}
\begin{tabular}{>{\hspace{-1mm}\small}p{2cm}<{\hspace{-1mm}}>{\hspace{-2mm}}r}
\toprule
\multicolumn{2}{c}{\small \bf \makecell{MaxMatch-Dropout $p{=}0.3$}} \\
\midrule
\texttt{different}       & 69.95\% \\
\texttt{differ ent}      & 14.74\% \\
\texttt{differ en t}     & 4.44\% \\
\texttt{diff ere nt}     & 3.05\% \\
\texttt{differ e nt}     & 1.36\% \\
\texttt{diff ere n t}    & 1.27\% \\
\texttt{dif fer ent}     & 0.94\% \\
\texttt{diff er ent}     & 0.94\% \\
\texttt{differ e n t}    & 0.56\% \\
\texttt{diff e rent}     & 0.41\% \\
\bottomrule
\end{tabular}
}

\bigskip

\resizebox{\linewidth}{!}{
\begin{tabular}{>{\hspace{-1mm}\small}p{2cm}<{\hspace{-1mm}}r}
\toprule
\multicolumn{2}{c}{\small \bf \makecell{BPE-Dropout $p{=}0.1$}} \\
\midrule
\texttt{together}        & 88.88\% \\
\texttt{to ge ther}      & 9.94\% \\
\texttt{to get her}      & 0.95\% \\
\texttt{tog ether}       & 0.08\% \\
\texttt{to ge t her}     & 0.07\% \\
\texttt{to g ether}      & 0.03\% \\
\texttt{to ge th er}     & 0.03\% \\
\texttt{tog e ther}      & 0.00\% \\
\texttt{t og ether}      & 0.00\% \\
\texttt{tog eth er}      & 0.00\% \\
\bottomrule
\end{tabular}
\begin{tabular}{>{\hspace{-1mm}\small}p{2cm}<{\hspace{-1mm}}>{\hspace{-2mm}}r}
\toprule
\multicolumn{2}{c}{\small \bf \makecell{MaxMatch-Dropout $p{=}0.3$}} \\
\midrule
\texttt{together}        & 69.99\% \\
\texttt{tog ether}       & 14.68\% \\
\texttt{tog eth er}      & 3.12\% \\
\texttt{to get her}      & 3.05\% \\
\texttt{t og ether}      & 1.35\% \\
\texttt{tog eth e r}     & 1.33\% \\
\texttt{to ge ther}      & 0.95\% \\
\texttt{tog et her}      & 0.94\% \\
\texttt{to get he r}     & 0.90\% \\
\texttt{t o get her}     & 0.41\% \\
\bottomrule
\end{tabular}
\hspace{5mm}
\begin{tabular}{>{\hspace{-1mm}\small}p{2cm}<{\hspace{-1mm}}r}
\toprule
\multicolumn{2}{c}{\small \bf \makecell{BPE-Dropout $p{=}0.1$}} \\
\midrule
\texttt{happening}       & 95.11\% \\
\texttt{happen ing}      & 1.93\% \\
\texttt{ha pp ening}     & 1.71\% \\
\texttt{happ ening}      & 0.91\% \\
\texttt{ha pp en ing}    & 0.16\% \\
\texttt{happ en ing}     & 0.08\% \\
\texttt{h app en ing}    & 0.03\% \\
\texttt{happ e ning}     & 0.03\% \\
\texttt{ha pp e ning}    & 0.02\% \\
\texttt{h app ening}     & 0.01\% \\
\bottomrule
\end{tabular}
\begin{tabular}{>{\hspace{-1mm}\small}p{2cm}<{\hspace{-1mm}}>{\hspace{-2mm}}r}
\toprule
\multicolumn{2}{c}{\small \bf \makecell{MaxMatch-Dropout $p{=}0.3$}} \\
\midrule
\texttt{happening}       & 70.05\% \\
\texttt{happen ing}      & 14.80\% \\
\texttt{happen in g}     & 4.46\% \\
\texttt{happ ening}      & 4.26\% \\
\texttt{happen i n g}    & 1.83\% \\
\texttt{ha pp ening}     & 0.94\% \\
\texttt{happ en ing}     & 0.89\% \\
\texttt{h app ening}     & 0.44\% \\
\texttt{happ e ning}     & 0.40\% \\
\texttt{ha p pe ning}    & 0.27\% \\
\bottomrule
\end{tabular}
}

\captionof{example}{
 Frequencies of tokenizations of several words sampled from BPE-Dropout (with $p = 0.1$) and MaxMatch-Dropout (with $p = 0.3$). The top row in each is the canonical tokenization.
}
\label{ex:bpe_dropout_20_full}
\end{table}

\appendix

\onecolumn

\end{document}